\documentclass[letterpaper, 10 pt, journal, twoside]{IEEEtran}

\IEEEoverridecommandlockouts                              

\usepackage{graphics} 
\usepackage{epsfig} 
\usepackage{amsmath} 
\usepackage{amssymb}  
\usepackage{balance}
\usepackage{algorithm}
\usepackage{algpseudocode}

\usepackage{changes}
\newcommand{\stkout}[1]{\ifmmode\text{\sout{\ensuremath{#1}}}\else\sout{#1}\fi}
\setdeletedmarkup{\stkout{#1}}

\newtheorem{theorem}{Theorem}
\newtheorem{definition}{Definition}
\newtheorem{proposition}{Proposition}
\newtheorem{lemma}{Lemma}
\usepackage{hyperref}

\begin{document}

\title{Decentralized Ergodic Control: Distribution-Driven Sensing and Exploration for Multi-Agent Systems}

\author{Ian Abraham and Todd D. Murphey
\thanks{This material is based upon work supported by the National Science Foundation under awards IIS-1426961 and IIS-1717951. Any opinions, findings, and conclusions or recommendations expressed in this material are those of the author(s) and do not necessarily reflect the views of the National Science Foundation.}
\thanks{Authors are with the Neuroscience and Robotics Laboratory (NxR) at the Department of Mechanical Engineering, Northwestern University, 2145 Sheridan Road Evanston, IL 60208 USA.}
\thanks{{\tt\small Email: i-abr@u.northwestern.edu, t-murphey@northwestern.edu}}
 \thanks{Manuscript received April 19, 2005; revised August 26, 2015.}}

\markboth{IEEE Robotics and Automation Letters. Preprint Version. Accepted June, 2018}
{Abraham \MakeLowercase{\textit{et al.}}: Decentralized Ergodic Control}

\maketitle

\begin{abstract}
We present a decentralized ergodic control policy for time-varying area coverage problems for multiple agents with nonlinear dynamics.
Ergodic control allows us to specify distributions as objectives for area coverage problems for nonlinear robotic systems as a closed-form controller. 
We derive a variation to the ergodic control policy that can be used with consensus to enable a fully decentralized multi-agent control policy.
Examples are presented to illustrate the applicability of our method for multi-agent terrain mapping as well as target localization.
An analysis on ergodic policies as a Nash equilibrium is provided for game theoretic applications.
\end{abstract}


\IEEEpeerreviewmaketitle

\section{INTRODUCTION}
\label{sec:introduction}

\IEEEPARstart{I}{n} the task of exploration and area coverage, decentralized robot networks have been shown to improve the sensing capacity of mobile robots \cite{carmel1999exploration, dudek1996taxonomy, manss2016decentralized} while minimizing computation for individual robotic agents. 
Shifting the computation to the individual robotic agent becomes necessary as the size of the multi-agent network becomes large and coordination of the network becomes costly for a single computing unit to calculate~\cite{manss2016decentralized, viseras2014efficient, khamis2014adaptive, pei2014distributed}.
This is more of an issue if the coordination algorithm becomes more complex with each agent, making it more desirable to have local, decentralized computation that only relies on neighboring information. 
This is also true as the underlying task becomes complex and additional environmental considerations must then be taken into account.
In this paper, we present an algorithm for dynamic decentralized area coverage derived from ergodic control~\cite{mavrommatiTRO2017realTime} that admits nonlinearities in the dynamics of robot and is general to many applications of multi-agent coordination.

Ergodic control~\cite{miller2013trajectory,miller2016ergodic,  mathew2011metrics, mavrommatiTRO2017realTime, shell2006ergodic} enables area coverage for robotic agents with nonlinear dynamics that is general to many applications.
By specifying the ergodic metric for area coverage, it was shown that one can synthesize trajectories that maximally optimize the ergodic metric, resulting in persistent coverage,\footnote{In the sense that the robot is always in motion.} visitation of the entire exploration domain~\cite{miller2013trajectory, miller2016ergodic, abraham2017ergodic}, and resilience to distractors in localization tasks~\cite{miller2015optimalrange, mavrommatiTRO2017realTime}.
In~\cite{mavrommatiTRO2017realTime} it was shown that one can formulate the ergodic control algorithm as a centralized ergodic controller for multiple agents.
However, it has yet to be shown how one can decentralize the algorithm for use in larger, more complex, multi-agent systems where control decisions are made on an individual basis.
Thus, the contribution of this work is a formulation of ergodic control for a multi-agent network as a decentralized algorithm that, through consensus, solves various forms of persistent area coverage problems using the ergodic metric for agents with nonlinear dynamics. 

Existing work in multi-agent coordination addresses the problems of area coverage~\cite{cortes2004coverage, lee2015multirobot, miah2017generalized}, inclusion of sensor constraints~\cite{kantaros2015distributed, vander2015algorithms}, and localization and estimation~\cite{freundlich2018distributed,vander2015algorithms}.
While these methods address specific problems in decentralized coordination, none of these methods have been shown to be flexible enough to solve all the problems.
While we initially frame our algorithm for area coverage, we provide additional examples for target localization, terrain estimation, and coverage in corridors to show that our method can be generalized to other tasks seen in multi-agent coordination~\cite{vander2015algorithms, freundlich2018distributed} without the need to change the specification of our algorithm.
Moreover, our method is distinct from coverage algorithms that rely on Voronoi segmentation of the environment to make coordinated decisions~\cite{kantaros2015distributed, cortes2004coverage, miah2017generalized, lee2015multirobot, alireza2011decentralized}.
Voronoi segmentation requires the specification of a metric for generation of the segmentation in addition to a metric for control and area coverage of each individual robotic agent.
When the dynamics of the robot are nonlinear, control synthesis requires additional assumptions or metrics~\cite{alireza2011decentralized, lee2015multirobot}.
Our method only uses the ergodic metric to formulate control for nonlinear dynamics~\cite{mavrommatiTRO2017realTime}.
Moreover, one can specify the ergodic metric with respect to information densities based on measurement models that include sensor physics/constraints~\cite{mavrommatiTRO2017realTime,miller2016ergodic}.
We show in Section~\ref{subsec:decentralized-ergodic-control-using-consensus} that the requirement of our decentralized algorithm is that the agents need only communicate coefficients representing their actions in order to make independent decisions that reduce the ergodic metric.

The outline of the paper is as follows: 
Section~\ref{sec:prob_statement} defines the problem of area coverage for multi-agent networks. 
Section \ref{sec:decentralized-ergodic-control} introduces ergodicity and the ergodic metric as well as formulates the ergodic control problem for decentralized multi-agent systems. 
A game theoretic analysis on ergodic control policies is provided in Section~\ref{sec:ergodic-control-policies-as-nash-equilibrium-strategies}.
Section \ref{sec:terrain-mapping-using-ergodic-area-coverage} demonstrates the algorithm on an area-coverage problem for multi-agents. 
We then present the problem for multi-agent target localization in Section~\ref{sec:ergodic-control-for-multi-agent-pursuit-evasion-games} and the conclusion is in Section~\ref{sec:conclusions}.

\section{Multi-Agent Area Coverage}
\label{sec:prob_statement}
In this section we present the problem statement that our method solves.
Let us consider a set of $N$ heterogeneous robotic agents where the evolution of the $i^\text{th}$ agent's state $x_i(t) : \mathbb{R}^+ \to \mathbb{R}^n$  at time $t\in \mathbb{R}^+$ is governed by the deterministic nonlinear equation
\begin{equation}
\dot{x}_i = f_i(x_i(t), u_i(t))
\end{equation}
where $u_i(t) : \mathbb{R}^+ \to \mathbb{R}^m$ is an applied control and $f_i(x,u) : \mathbb{R}^n \times \mathbb{R}^m \to \mathbb{R}^n$ is a nonlinear function.
Furthermore, let us define a bounded domain $\mathcal{X}_v \subset \mathbb{R}^v$ whose limits are $\left[ 0, L_1\right]\times \left[ 0, L_2\right] \times \ldots \left[ 0, L_v\right]$ with $v \le n$.
We can consider this bounded domain a ``search space'' where we can define any arbitrary spatial statistic $\phi(s) : \mathcal{X}_v \to \mathbb{R}^+$\footnote{Under the assumption that $\int_{\mathcal{X}_v} \phi(s) ds = 1$.} where $s \in \mathbb{R}^v \subset \mathbb{R}^n$.
Typically $\phi(s)$ is generated from the expected information density~\cite{miller2016ergodic,mavrommatiTRO2017realTime} based on the measurement model and sensor constraints.
The goal of multi-agent area coverage is to position the agents in a system in such a manner that the states of the system $x(t) = \left[x_1(t)^\top, \ldots, x_N(t)^\top \right]^\top$ are proportional to the spatial statistics $\phi(s)$.
That is, we want the statistics of the trajectory of the robots, which we will define as $c(s, x(t))$ to be equal to the spatial statistics $\phi(s)$ through some metric (in our case ergodicity).

We note that our approach treats the problems of target tracking, estimation, and area coverage as the same problem of persistent area coverage, that is, the spatial statistics contains the information for all these problems (which we specify in Sections~\ref{sec:ergodic-control-for-multi-agent-pursuit-evasion-games} and~\ref{sec:terrain-mapping-using-ergodic-area-coverage}).
Moreover, we emphasize persistent area coverage because the basis of the ergodic metric (see Section~\ref{subsec:ergodicity-and-the-ergodic-metric}) revolves around the time-averaged statistics of the multi-agent system trajectories in the search space.
This results in persistent movement and monitoring, rather than placement, of an agent.

The following section formulates the decentralized ergodic controller for a multi-agent system.

\section{Decentralized Ergodic Control}
\label{sec:decentralized-ergodic-control}

In this section, ergodicity and the ergodic metric are introduced and we formulate an ergodic control policy for multi-agent systems.
We make note of the terminology {\it distributed} and {\it decentralized} used in this paper as two distinct terms: 

\begin{definition}
 A {\it distributed} algorithm is one where the initialization of the optimization occurs in a centralized computer hub and then the calculation for the optimization are offloaded onto a set of individual computation units.
\end{definition}
\begin{definition}
 A {\it decentralized} algorithm is one where each individual computational unit solves their own optimization problem that, through communication with a network, solves a larger global optimization problem (typically using some form of consensus)~\cite{deo2016graph, bertsekas1989parallel}.
\end{definition}

\subsection{Ergodicity and the Ergodic Metric}
\label{subsec:ergodicity-and-the-ergodic-metric}
Assume the state at time $t$ is given by $x(t) : \mathbb{R}^+ \to \mathbb{R}^n$. 
Controls to the robot at time $t$ are $u(t) : \mathbb{R}^+ \to \mathbb{R}^m$. \footnote{We drop the $i^\text{th}$ indexing notation for readability and to illustrate that the multi-agent system can be treated as a larger, unified system in later sections.}
The dynamics of the robot are assumed to be governed by a control-affine dynamical system of the form
\begin{equation} \label{eq:robot_dynamics}
\dot{x}(t) = f(x(t),u(t)) = g(x(t)) + h(x(t)) u(t)
\end{equation}
where $g(x) : \mathbb{R}^n \to \mathbb{R}^n$ is the free, unactuated dynamics of the robot, and $h(x): \mathbb{R}^n \to \mathbb{R}^{n \times m}$ is the dynamic control response subject to input $u(t)$.
Let us consider the robot's time-averaged statistics $c(s, x(t))$ for a trajectory $x(t)$ (i.e., the statistics describing where the robot spends most of its time) for some time interval $t \in \left[ t_i, t_i + T\right]$ as
\begin{equation}\label{eq:time_avg_stats}
c(s, x(t)) = \frac{1}{T}\int_{t_i}^{t_i+T} \delta (s - x_v(t)) dt,
\end{equation}
where $\delta$ is a Dirac delta function, $T \in \mathbb{R}^+$ is the time horizon, $t_i \in \mathbb{R}^+$ is the $i^\text{th}$ sampling time, and $x_v(t) \in \mathbb{R}^v$ is the state that intersects with the search space.
An ergodic metric~\cite{mathew2011metrics} which relates the two distributions $c(s,x(t))$ and $\phi(s)$ is:
\begin{align} \label{eq:ergodic_metric}
\mathcal{E}(x(t)) & = q \,\sum_{k \in \mathbb{N}^v} \Lambda_k \left(c_k -\phi_k \right)^2   \\
& = q \, \sum_{k \in \mathbb{N}^v} \left( \frac{1}{T} \int_{t_i}^{t_i + T} F_k(x(t)) dt - \phi_k \right)^2 \nonumber
\end{align}
where
\begin{equation*}
\phi_k =  \int_{\mathcal{X}_v} \phi(s) F_k(s) ds, 
\end{equation*} 
$q \in \mathbb{R}^+$ is a scalar weight on the metric, and $c_k$ are the Fourier decompositions\footnote{The cosine basis function is used, however, any choice of basis function $F_k$ can be used.} of $c(s,x(t))$ and $\phi(s)$ with
\begin{equation*}
F_k(x) = \frac{1}{h_k}\prod_{i=1}^v \cos \left( \frac{k_i \pi x_i}{L_i} \right)
\end{equation*}
being the cosine basis function for a given coefficient $k \in \mathbb{N}^v$,  $h_k$ is a normalization factor defined in~\cite{mathew2011metrics}, and $\Lambda_k = (1 + \Vert k \Vert^2)^{-\frac{v+1}{2}}$ are weights on the frequency coefficients.
A robot whose control inputs result in a trajectory $x(t)$ that minimizes (\ref{eq:ergodic_metric}) as $t\to \infty$ is then said to be optimally ergodic with respect to the target distribution.

Because we are computing the ergodic control in receding horizon, and the target distribution $\phi(s)$ can be time-varying, a history of where a robot has been is maintained in memory in order to compute the ergodic metric.
The ergodic metric is then computed by adding a time parameter $\Delta t_\mathcal{E}$ which governs how far into the past the robot must remember where it has been. 
Equation (\ref{eq:ergodic_metric}) then becomes
\begin{equation}\label{eq:aug_ergodic_metric}
\mathcal{E}(x(t)) = q \, \sum_{k \in \mathbb{N}^v} \left( \frac{1}{T_\mathcal{E}} \int_{t_i-\Delta t_\mathcal{E} }^{t_i + T} F_k(x(t)) dt - \phi_k\right) ^2.
\end{equation}	
Note that choosing $\Delta t_\mathcal{E} = t_i$ would result in storing all past states.
This can be avoided by recursively defining $c_k$ as shown in~\cite{mavrommatiTRO2017realTime}.
In addition, choosing a $\Delta t_\mathcal{E} < T$ would result in very myopic behavior (i.e., only spending time in regions of high spatial statistics).
This is often desired if a time-varying spatial distribution $\phi(s,t)$ is specified where past information is rendered uninformative as the underlying spatial statistics change rapidly.
In practice, a choice of $\Delta t_\mathcal{E} = 2T$ is empirically a reasonable start which can be tuned to performance needs after further evaluation.

\subsection{Ergodic Control}
\label{subsec:ergodic-control}
In~\cite{miller2013trajectory} the ergodic controller is formulated using a trajectory optimization scheme.
While this approach does give optimal solutions, it is difficult for the controller to run in real time. 
As a result,~\cite{mavrommatiTRO2017realTime} developed a hybrid systems approach using~\cite{ansari2016sequential} to obtain control policies that sufficiently reduce the ergodic metric.
We formulate our controller using a similar approach, but provide a variation to the controller that allows the policy to be fully distributable. 
	
Rather than directly minimizing (\ref{eq:ergodic_metric}) with respect to $x(t)$ and $u(t)$, we consider the sensitivity of (\ref{eq:ergodic_metric}) with respect to an infinitesimal time of application  $\lambda \in \mathbb{R^+} \to 0$ of the best possible control $u_\star(t) : \mathbb{R}^+ \to \mathbb{R}^m$ that sufficiently reduces (\ref{eq:ergodic_metric}) at time $\tau \in \mathbb{R}^+$ from some default control $u_\text{def}(t) : \mathbb{R}^+  \to \mathbb{R}^m$.
Following~\cite{mavrommatiTRO2017realTime}, we take the derivative of (\ref{eq:ergodic_metric}) with respect to the duration time $\lambda$ of control $u_\star(t)$ which gives the sensitivity (known as the mode insertion gradient~\cite{vasudevan2013consistent, axelsson2008gradient, egerstedt2006transition,caldwell2016projection}).
\begin{proposition} \label{prop:mode_insert}
The first order sensitivity of (\ref{eq:ergodic_metric} with respect to the control duration $\lambda$ of the applied control $u_\star(\tau)$ is
\begin{equation}\label{eq:mode_insertion}
\frac{\partial \mathcal{E}}{\partial \lambda} \Big \vert_\tau = \rho(\tau)^\top (f_2(\tau, \tau) - f_1(\tau))
\end{equation}
where $f_2(t, \tau) = f(x(t), u_\star(\tau))$, $f_1(t) = f(x(t), u_\text{def}(t))$, and $\rho(t) : \mathbb{R}^+ \to \mathbb{R}^n$ is given by the differential equation
\begin{equation*}
\dot{\rho} = - 2 \frac{q}{T} \sum_{k \in \mathbb{N}^v} \Lambda_k (c_k - \phi_k) \frac{\partial F_k}{\partial x} - \frac{\partial f}{\partial x}^\top \rho(t)
\end{equation*}
with $\rho(t_i + T) = \bold{0} \in \mathbb{R}^n$.
\end{proposition}
\begin{proof}
See~\cite{mavrommatiTRO2017realTime} for more details.
\end{proof}
The mode insertion gradient now represents the sensitivity of the ergodic metric with respect to an application of a control $u_\star (t)$. 

Given the mode insertion gradient, we seek to find the control $u_\star(t)$ that most significantly decreases in the objective (\ref{eq:ergodic_metric}).
We can write this as an unconstrained optimization problem of the form
\begin{equation}\label{eq:secondary_objective}
J_2 = \int_{t_i}^{t_i + T} \frac{\partial \mathcal{E}}{\partial \lambda}\Big \vert_t + \frac{1}{2}\Vert u_\star(t) - u_\text{def}(t) \Vert_R^2
\end{equation}
where $R \in \mathbb{R}^{m \times m}$ is a positive definite matrix that weighs $u_\star(t)$.
Note that~(\ref{eq:secondary_objective}) is quadratic in $u_\star(t)$ which encodes a regularization term with respect the default control $u_\text{def}$ and includes a cost on sufficient decrease in the mode insertion gradient. 
The minimizer of~(\ref{eq:secondary_objective}) with respect to $u_\star(t)$ is the control that provides the most negative mode insertion gradient and reduces the objective (\ref{eq:ergodic_metric}).
\begin{proposition}
The solution to $u_\star(t)$ that minimizes (\ref{eq:secondary_objective}) is
\begin{equation} \label{eq:ustar}
u_\star(t) = -R^{-1} h(x)^\top \rho(t) + u_\text{def}(t).
\end{equation}
\end{proposition}
\begin{proof}
Taking the derivative of (\ref{eq:secondary_objective}) with respect to control $u_\star(t)$ and setting the solution to zero gives
\begin{align}\label{eq:secondary_objective2}
\frac{\partial J_2}{\partial u_\star} & = \int_{t_i}^{t_i + T} \frac{\partial }{\partial u_\star} \left( \frac{\partial \mathcal{E}}{\partial \lambda}\right) + R(u_\star - u_\text{def}) dt \nonumber \\
& = \int_{t_i}^{t_i + T}  h(x)^\top\rho + R(u_\star - u_\text{def}) dt = 0
\end{align}
where the dependency on time is dropped for simplicity.
Solving for $u_\star$ in (\ref{eq:secondary_objective2}) gives
\begin{equation*}
u_\star(t) = -R^{-1}h(x(t))^\top\rho(t) + u_\text{def}(t).
\end{equation*}
\end{proof}
\begin{lemma}\label{lemma:mode_insert}
Assuming that $h(x)^\top \rho \neq 0$, the mode insertion gradient in (\ref{eq:mode_insertion}) is always negative for $u_\star(t)$ defined in (\ref{eq:ustar}), that is $\frac{\partial \mathcal{E}}{\partial \lambda} < 0 \,  \forall u_\star \in \mathcal{U}$ where $\mathcal{U}$ is the control space.
\end{lemma}
\begin{proof}
Inserting (\ref{eq:ustar}) into (\ref{eq:mode_insertion}) gives
\begin{align}\label{eq:mode_insertion_neg}
\frac{\partial \mathcal{E}}{\partial \lambda} &= h(x)^\top\rho \left(-R^{-1} h(x)^\top\rho \right) \nonumber \\
& = -\rho^\top h(x)R^{-1}h^\top\rho = - \Vert h(x)^\top\rho \Vert_{R^{-1}}^2 < 0.
\end{align}
Thus (\ref{eq:mode_insertion_neg}) shows us that $\forall u_\star \in \mathcal{U}$ defined in (\ref{eq:ustar}), $\frac{\partial \mathcal{E}}{\partial \lambda}<0$.
\end{proof}
Because (\ref{eq:ustar}) always provides a negative $\frac{\partial \mathcal{E}}{\partial \lambda}$, this implies that each control that is chosen will result in a decrease in (\ref{eq:ergodic_metric}); thus eventually minimizing the ergodic metric.
Additionally, as in~\cite{mavrommatiTRO2017realTime}, a contractive constraint on the reduction of the ergodic metric is enforced that further provides a reduction in the ergodic metric from the previous control calculation time.

In many robotics applications, it is required that the control is saturated due to actuation limits in the robot while maintaining some form of sufficient decrease in the objective cost. 
In this work, we select a time of application $\tau$ that results in the most negative mode insertion gradient, or more formally written by
\begin{equation*}
\tau_\star= \underset{\tau}{\text{argmin }} \frac{\partial \mathcal{E}}{\partial \lambda}
\end{equation*}
where the subscript $\star$ indicates the time of application that results in the most negative mode insertion gradient.
A line search~\cite{armijo1966minimization} is then used to find the duration $\lambda$ that significantly reduces (\ref{eq:ergodic_metric}) subject to the saturated control $u_\star(\tau)$.
The resulting control is then added to the default control $u_\text{def}(t) = u_\star(\tau) \forall t \in \left[ \tau, \tau+\lambda \right] \cap \left[ t_i, t_i+t_s\right]$ where $t_s$ is the sampling time and $u_\star(\tau)$ is saturated.

The following subsection derives the ergodic control policy for a decentralized multi-agent systems.

\subsection{Decentralized Ergodic Control using Consensus}
\label{subsec:decentralized-ergodic-control-using-consensus}

Consider a set of $N$ agents with state $x(t) = \left[ x_1(t)^\top, x_2(t)^\top, \ldots, x_N(t)^\top\right]^\top : \mathbb{R}^+ \to \mathbb{R}^{n N}$. \footnote{For readability we consider a homogeneous set of agents with the same state dimension $x_i(t) \in \mathbb{R}^n$. However, this analysis can be done for a heterogeneous set of agents with arbitrary dynamics and state dimensions.}
\begin{proposition}
Given the default trajectory of each agent $x(t) \forall t \in \left[ t_i - \Delta t_\mathcal{E}, t_i + T\right]$ subject to $u_\text{def}(t)$, the control policy (\ref{eq:ustar}) is {\it distributable} amongst each individual agent and independent of the other agent's control policy.
\end{proposition}
\begin{proof}
Let us first define the dynamics of the collective multi-agent system as
\begin{align} \label{eq:collective_dynamics}
\dot{x} & = f(x,u) = g(x) + h(x) u \nonumber\\
& = \begin{bmatrix}
g_1(x_1) \\
g_2(x_2) \\
\vdots \\
g_N(x_N)
\end{bmatrix} + 
 \begin{bmatrix}
h_1(x_1) & \ldots & 0\\
\vdots& \ddots & \\
0 & & h_N(x_N)
\end{bmatrix} u
\end{align}
where $h(x)$ is block diagonal.
The multi-agent system's contribution to the time-averaged statistics $c_k$ can be rewritten as
\begin{align}\label{eq:centralized_ck}
c_k & = \frac{1}{N} \sum_{j=1}^N \frac{1}{T_\mathcal{E}} \int_{t_i -\Delta t_\mathcal{E}}^{t_i + T} F_k(x_j(t)) dt \nonumber \\
& = \frac{1}{T_\mathcal{E}} \int_{t_i - \Delta t_\mathcal{E}}^{t_i+T} \tilde{F}_k(x(t)) dt
\end{align}
where $\tilde{F}_k(x(t)) = \frac{1}{N}\sum_j F_k(x_j(t))$.
The mode insertion gradient (\ref{eq:mode_insertion}) under a multi-agent dynamical system now has $f_1(t)$ and $f_2(t, \tau)$ defined by (\ref{eq:collective_dynamics}) and the convolution equation for the adjoint variable $\rho(t)$ becomes
\begin{equation}\label{eq:decentralized_adjoint}
\dot{\rho} = -2 \frac{q}{T_\mathcal{E}}\sum_{k \in \mathbb{N}^v} \Lambda \left( c_k - \phi_k \right) \frac{\partial \tilde{F}_k}{\partial x} - \frac{\partial f}{\partial x}^\top \rho
\end{equation}
where 
\begin{equation*}
\frac{\partial \tilde{F}_k}{\partial x} = \frac{1}{N} \begin{bmatrix}
\frac{\partial F_k (x_1)}{\partial x_1} \\
\vdots \\
\frac{\partial F_k(x_N)}{\partial x_N}
\end{bmatrix} 
\text{ and }
\frac{\partial f}{\partial x} =
\begin{bmatrix}
\frac{\partial f_1}{\partial x_1} & 0 & \ldots & 0 \\
0 & \frac{\partial f_{2}}{\partial x_{2}} \\
\vdots & & \ddots  & \\
0 &  & & \frac{\partial f_N}{\partial x_N}
\end{bmatrix}
\end{equation*}
is block diagonal.
Because each agent's dynamics are independent of each other, (\ref{eq:decentralized_adjoint}) can be written independently for each agent as
\begin{equation*}
\dot{\rho}_j = -2\frac{q}{T_\mathcal{E} N} \sum_{k \in \mathbb{N}^v} \Lambda_k (c_k - \phi_k) \frac{\partial F_k(x_j)}{\partial x_j} - \frac{\partial f_j}{\partial x_j}^\top \rho_j.
\end{equation*}
Similarly, the ergodic control policy derived from (\ref{eq:secondary_objective2}) becomes
{\small
\begin{multline} \label{eq:expanded_control}
\begin{bmatrix}
u_{\star,1} (t) \\
\vdots \\
u_{\star,_N}(t) \\
\end{bmatrix}
=
-R^{-1}  
 \begin{bmatrix}
h_1(x_1) & \ldots & 0\\
\vdots& \ddots & \\
0 & & h_N(x_N)
\end{bmatrix} ^\top
\begin{bmatrix}
\rho_1 (t) \\
\vdots \\
\rho_N(t)
\end{bmatrix}
 \\+
 \begin{bmatrix}
 u_{\text{def},1} (t) \\
 \vdots \\
 u_{\text{def},N}(t)
 \end{bmatrix}
\end{multline}}
where $R\in\mathbb{R}^{mN \times mN}$ and $mN$ is the size of the collective multi-agent system control input.
Since $h(x)$ is block diagonal,  (\ref{eq:expanded_control}) becomes
\begin{equation}\label{eq:policy_independence}
u_{\star,j} (t) = -R_j^{-1} h_j(x_j)^T \rho_j(t) + u_{\text{def},j}(t) 
\end{equation}
for each agent $j \in \left[ 1, \ldots, N \right]$ and $R_j \in \mathbb{R}^{m \times m}$.
The control policy in (\ref{eq:policy_independence}) for the $j^\text{th}$ agent does not depend on the $i^\text{th}$ agent and therefore is distributable.
\end{proof}

While the $j^\text{th}$ control policy is independent of the $i^\text{th}$ control policy, it is assumed starting from (\ref{eq:aug_ergodic_metric}) that each agent's past and anticipated trajectory is known to all agents before calculating the control policy.
We can consider this a distributed ergodic control policy where the control computation is still done on individual CPUs on-board the agents, but the initial conditions are required to be sent from a central communication hub. 
Instead of a distributed controller, we seek to completely remove the need for a centralized communication hub and have fully independent agents solve smaller ergodic control problems that solve the same larger multi-agent ergodic control problem. 
We address this problem using consensus-based methods where a network of agents communicates with one another the local $c_k$ for the individual agent. 

Rather than communicating the past and anticipated trajectories of each agent (which may have large dimensionality) in the network, we communicate the $c_k$ values instead. \footnote{It is assumed that each agent has the same $\phi_k$ target, however, the same analysis can be done to form a consensus on the target $\phi_k$ values.}
\begin{proposition}
A connected multi-agent network under consensus over the $c_k$ coefficients approximates the time-average statistics $c_k$ of the centralized ergodic metric (\ref{eq:centralized_ck}), that is $\tilde{c}_k \to c_k$ as $t \to \infty$ where $\tilde{c}_k$ is the consensus-based time-average statistics.
\end{proposition}
\begin{proof}
Consider the collective time-averaged statistics $c_k$ for the system in (\ref{eq:centralized_ck}):
\begin{equation*}
c_k = \frac{1}{N} \sum_{j=1}^N \frac{1}{T} \int_{t_i}^{t_i + T} F_k(x_j(t))dt.
\end{equation*}
Equation (\ref{eq:centralized_ck}) is simply an averaging of the individual agent's spatial statistics. 
Let us then define a row and column stochastic consensus matrix $P$ (e.g., $\sum_j P_{ij} = \bold{1}$) that defines the network connectivity amongst the agents \cite{deo2016graph, bertsekas1989parallel}.
The operation $\sum_j P_{ij} c_{k,j}$ is equivalent to taking an average of the local $c_{k,j}$ values for each neighboring agent. \footnote{For simplicity in notation, we assume that $P_{ij}$ refers to a block matrix such that $P \in \mathbb{R}^{\vert k \vert N \times \vert k \vert N}$ and $P_{ij} \in \mathbb{R}^{\vert k \vert \times \vert k \vert}$ where $\vert k \vert$ is the total number of $c_k$ coefficients.} 
Therefore, we can write a consensus on the collective $c_k$ (\ref{eq:centralized_ck}) using $P$ as \cite{deo2016graph, bertsekas1989parallel}
\begin{equation*}
\lim_{t_k \to \infty}\sum_j P_{ij}^{t_k} c_{k,j} = \frac{1}{N}\sum_j \frac{1}{T} \int_{t_0}^{t_0 +T} F_k(x_j(t)) dt
\end{equation*}
where $N$ is the number of agents, $t_k$ is the number of times that $P_{ij}c_{k,j}$ values have been communicated through the network and averaged.
Thus consensus amongst all the agents approximates the collective multi-agent system time-averaged statistics $c_k$ in (\ref{eq:centralized_ck}).
\end{proof}
Algorithm~\ref{alg:decentralized-ergodic-control} is provided to illustrate the decentralized ergodic control policy for multi-agent systems.

\begin{algorithm}
\caption{Decentralized Ergodic Control} \label{alg:decentralized-ergodic-control}
\centering
\begin{algorithmic}[1]
\State \textbf{initialize:} agents $N$ with initial condition $x_j(0)$, initial target distribution $\phi_{k,0}$, $t_0, t_f, t_s$, time horizon $T$, ergodic memory $\Delta t_\mathcal{E}$ and network $P$.
\While{$t_i < t_f$}
\For{each agent}
\Comment Control step
\State simulate $x(t), \rho(t)$ for $t \in \left[ t_i, t_i +T \right] $ from $x(t_i)$
\State compute $u_\star(\tau)$ from (\ref{eq:ustar})
\State calculate $\tau$ and $\lambda$ from \cite{ansari2016sequential, mavrommatiTRO2017realTime}
\State $u_\text{def}(t) = u_\star(\tau) \forall t \in \left[ \tau, \tau+\lambda \right] \cap \left[ t_i, t_i+t_s\right]$
\EndFor
\For{each agent}
\Comment Communication Loop
\State Send $c_{k,j}$ to $i^\text{th}$ neighbors in the network $P$
\State Receive $c_{k,i}$ from neighbors and average amongst $i^\text{th}$ neighbors
\EndFor
\State apply control $u_\text{def}(t_i)$
\State $i \gets i +1$
\EndWhile
\end{algorithmic}
\end{algorithm}

\subsection{Communication Complexity and Scalability}
Since the ergodic metric is defined in terms of the Fourier coefficients of the agent's trajectory and the spatial statistics, each agent is only required to transmit their own $c_{k,j}$ trajectory coefficients.
The benefit of this is two-fold: First, each agent in the decentralized network need only store their own past trajectory information for computing $c_{k,j}$.
Thus, the required storage for a $64$ bit memory is $64 * \Delta t_\mathcal{E}/ t_s * n$ bits where $t_s$ is the sampling rate.
We can further reduce the memory requirements by recursively defining the $c_{k,j}$ values as done in~\cite{mavrommatiTRO2017realTime}.
The second benefit is in the complexity and scale of the algorithm as the number of agents increases.
Since each agent only needs to communicate their local $c_{k,j}$ values to their neighbors, the computational burden lies in computing the ergodic control for the individual agents themselves. 
Because we have shown that we can fully decentralize the ergodic control calculations, the computation remains constant to each robot.
Thus, the computational complexity of the ergodic controller only scales with the dimensions of the single agent's state (which for practical purposes will remain constant as the agents' state dimensions are not time-varying) and the decentralized algorithm does not scale by increasing the number of agents in the network.

In the following section, we provide an analysis of the ergodic control policy in a game-theoretic point of view.

\begin{figure*}[thpb]
\centering
\framebox{\parbox{6.8in}{
\centering
\includegraphics[scale=0.9]{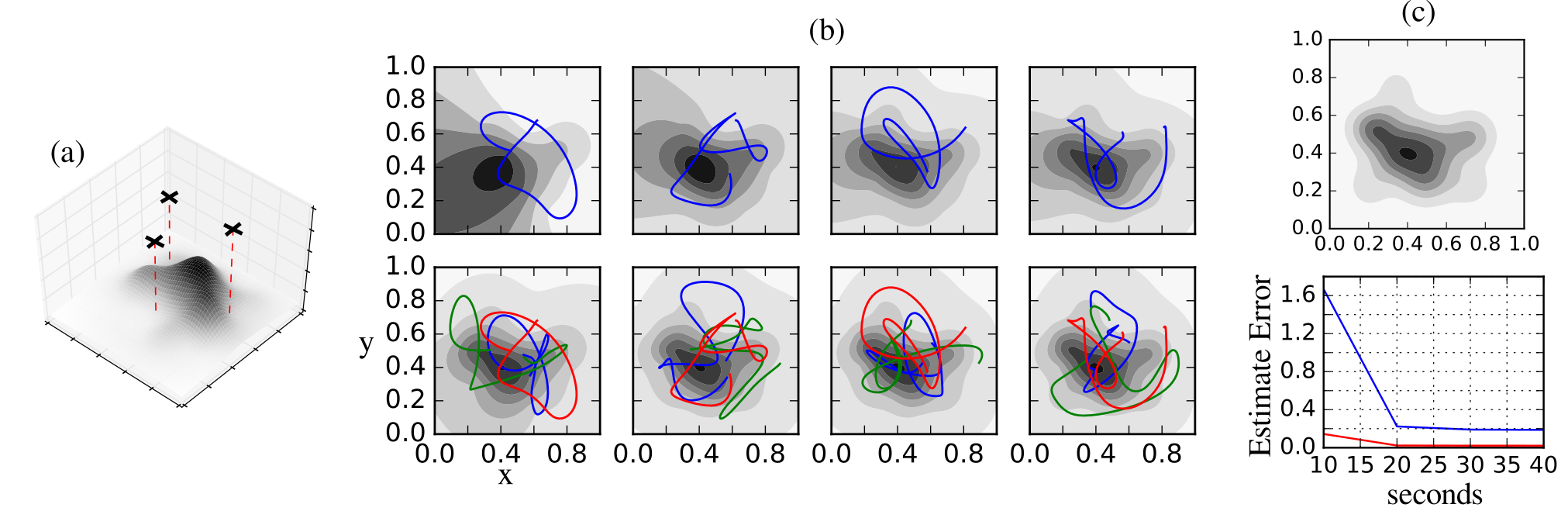}}}
\caption{
(a) Three quadcopter agents are depicted in the map with the terrain. The red dashed line indicates the location.
(b) Trajectories of a single agent (blue) and a multi-agent system (blue, red, green) are shown estimating terrain. The terrain map is obtained from height measurements (dark regions represent high elevation).
(c) Top-down orthographic view of the terrain map for comparison with the results in (b). 
Error of the estimate is shown at $10$ second intervals of collected data. 
Our algorithm is able to coordinate the agents such that more area is covered, enabling the collection of more data and a resulting better terrain map.}
\label{fig:mapping_results}
\end{figure*}

\section{Ergodic Control Policies as Nash Equilibrium Strategies}
\label{sec:ergodic-control-policies-as-nash-equilibrium-strategies}
In this section, we analyze the ergodic control policy from a game theoretic point of view in adversarial multi-agent games.
\begin{definition}
A game is defined by a tuple $(\mathcal{P}, \mathcal{A}, \mathcal{O}, \mu, \mathcal{U})$ where $\mathcal{P}$ is the set of players in a game, $\mathcal{A}$ is the set of control actions $u(t)$ where $u_i(t) \forall t \in \left[ t_i, t_i + T \right], \forall i \in \mathcal{P}$ is considered an action or strategy profile, $\mathcal{O}$ is the set of outcomes (or state trajectories in our case), $\mu : \mathcal{A} \to \mathcal{O}$ is the function that maps actions to outcomes (in our case this is the robot dynamics), and last $\mathcal{U} : \mathcal{O} \to \mathbb{R}$ is a utility function that we index for each $i^\text{th}$ player using the subscript $i$.
\end{definition}

Each agent is defined by $\mathcal{P}$.
The action profile or strategy $\mathcal{A}$ is defined by the ergodic control policy subject to a target distribution.
The resultant trajectory $x(t)$ for each agent is the outcome $\mathcal{O}$ subject to the actions passing through the dynamics $f(x,u)$ of the system ($\mu$).
Here, we treat the utility function $\mathcal{U}$ as the ergodic metric.
In game theory, the notion of Nash equilibrium~\cite{bhattacharya2009existence, myerson1978refinements} is often used to describe a strategy in a game.
\begin{definition}
A strategy is a Nash equilibrium if for each agent $i$, $\mathcal{U}_i(u) \le \mathcal{U}_i( u_{-i})$ where $u_{ -i}$ is the updated strategy profile for all agents not including agent $i$'s strategy.
\end{definition}

Nash equilibrium tells us whether a strategy results in the best possible expected utility of each agent subject to the other agents' actions.
We consider Nash equilibrium in the problem of target localization and evasion.
Specifically, we look at what strategy an evader can use to acquire a Nash equilibrium with the pursuer (localizer) (i.e., a game between the pursuer and evader while the pursuer expends energy not localizing the evader).
\begin{theorem}
A Nash equilibrium strategy against a pursuer with an ergodic policy is for the evader to adopt an ergodic policy.
\end{theorem}
\begin{proof}
Consider two agents, $a$ and $b$ on opposing sides of a game.
Agent $a$ is ergodic with respect to a target distribution $\phi_a(s)$. 
Agent $b$ is ergodic with respect to $\phi_b(s)$.
We assume that the target distribution of agent $a$ and $b$ is a function of the state of the agents, that is, $\phi_a(s) = \phi_a(s, x_a(t), x_b(t))$ and  $\phi_b(s) = \phi(s, x_a(t), x_b(t))$.
From Lemma~\ref{lemma:mode_insert}, we have shown that $\frac{\partial \mathcal{E}}{\partial \lambda} < 0 ,\ \forall u(t)$ defined by an ergodic policy.
As a result, as $t \to \infty$, both agents are asymptotically optimally ergodic with respect to their own target distributions so long as each action reduces the ergodic objective.
Therefore, we can write the change in the utility function\textemdash which we define as the ergodic metric\textemdash as
\begin{align*}
\mathcal{U}_i(x(t) \mid u) - \mathcal{U}_i(x(t) \mid u_{-i} ) & = \Delta U_i  \approx \frac{\partial \mathcal{E}_i}{\partial \lambda} \lambda < 0 \\
\mathcal{U}_i(x(t) \mid u) - \mathcal{U}_i(x(t) \mid u_{-i} )  & < 0 \\
\mathcal{U}_i(x(t) \mid u) &< \mathcal{U}_i(x(t) \mid u_{-i} )
\end{align*}

Thus, an ergodic control strategy is a Nash equilibrium strategy.
\end{proof}

This kind of analysis lends some insight towards formally viewing ergodic policies with respect to game theory and with application in general multi-agent games.

In the following section, we provide examples for typical uses of our proposed method for multi-agent area coverage problems and a comparison with an area coverage in corridors and tracking a time-varying distribution.

\section{Ergodic Area Coverage for Multi-Agent Elevation Mapping}
\label{sec:terrain-mapping-using-ergodic-area-coverage}
In this section we illustrate the capabilities of a decentralized ergodic controller for multi-agent area coverage for elevation mapping.
We use this example to show improved area coverage of a decentralized ergodic controller while comparing with a centralized controller for the same task.

\begin{figure}[thpb]
\centering
\framebox{\parbox{3.25in}{
\centering
\includegraphics[scale=1.0]{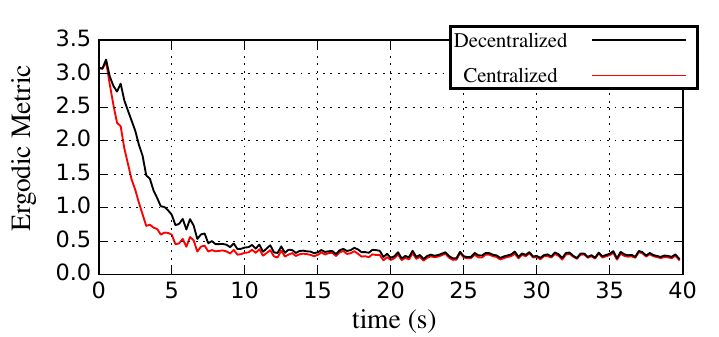}}}
\caption{Comparison of ergodic metric for a decentralized ergodic scheme versus a centralized ergodic scheme. 
Due to communication and consensus amongst the agent in a decentralized scheme, the ergodic metric does not reduce as quickly as a centralized scheme would.
However, the decentralized scheme is quick to reach consensus and performs comparably to the centralized scheme.}
\label{fig:coverage_comparison}
\end{figure}

\begin{figure}[thpb]
\centering
\framebox{\parbox{3.25in}{
\centering
\includegraphics[scale=0.9]{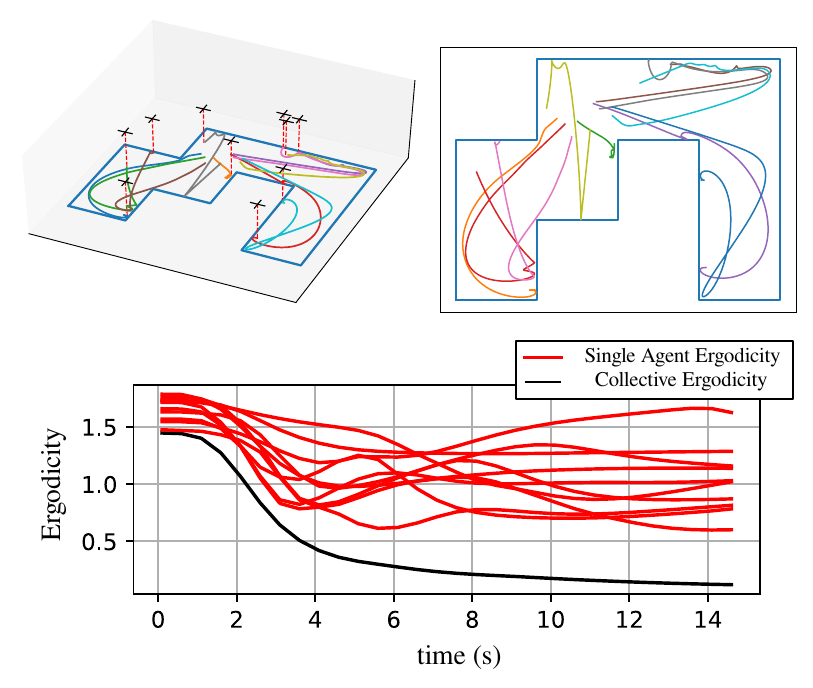}}}
\caption{
Comparison of area coverage in a corridor defined in~\cite{kantaros2015distributed} with small (point-wise) visibility for each agent. 
Ergodic control allows the agents to compensate for the limited sensing through motion while coordinating the decentralized agents to minimize the area coverage objective.
}
\label{fig:corridor_comparison}
\end{figure}

\subsection{Problem Setup}
A $12$ dimensional quadrotor~\cite{martin2010true} is used for the robotic agent dynamics with $4$ inputs directly controlling thrust, yaw, pitch, and roll angular accelerations.
Each agent measures ground elevation relative to the agent's altitude which it uses to construct a model of the terrain.
Three agents are used that are fully connected to one another.
The agents are randomly initialized and a Gaussian Process~\cite{vasudevan2009gaussian, plagemann2008learning} is used to construct the terrain elevation from data collected after $10$ second intervals.

\begin{figure*}[thpb]
\centering
\framebox{\parbox{6.8in}{
\centering
\includegraphics[scale=0.9]{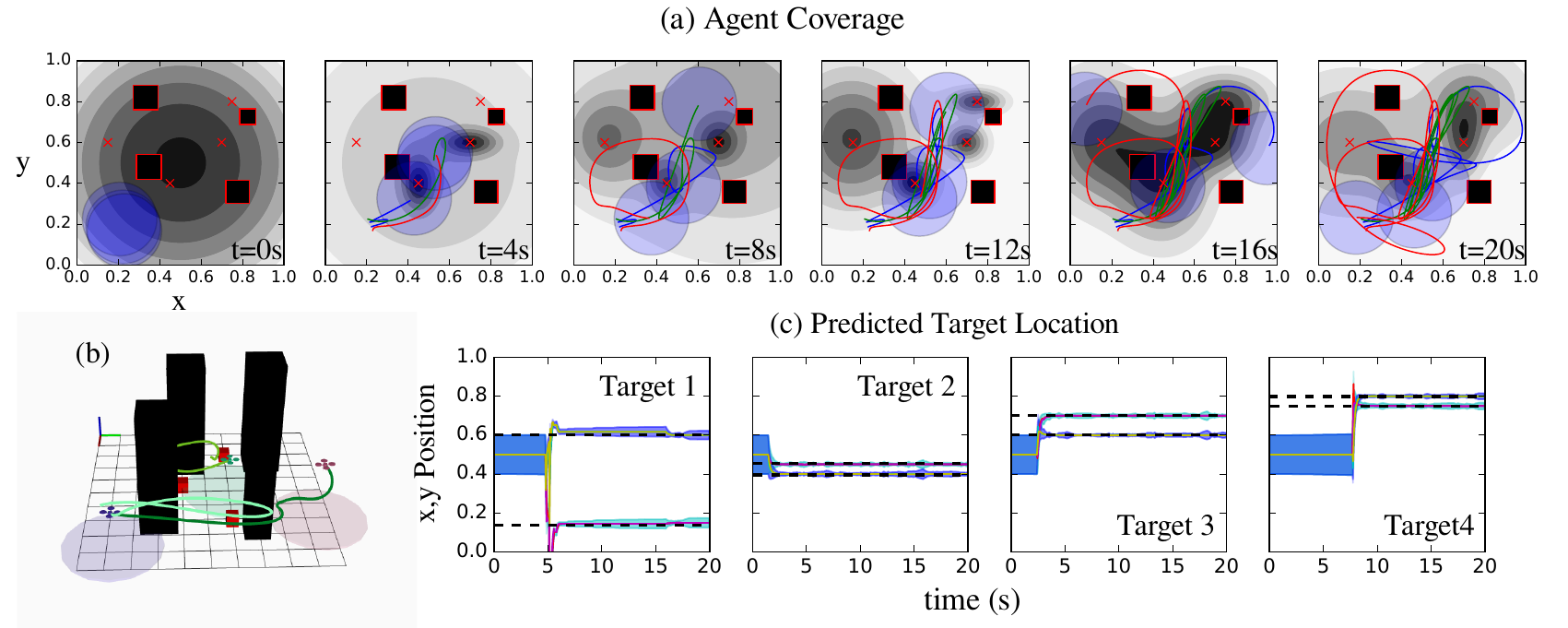}}}
\caption{
Target localization is illustrated using $3$ agents with deterministic quadcopter dynamics and $4$ unknown stationary targets depicted as the red crosses.
(a) Area coverage of the quadcopter is shown as the lines associated with each quadcopter  at distinct times.
The density function underneath shows the likely locations of the targets with darker regions indicating higher likelihood values.
Obstacles are shown as black squares with red outlines. A 3D rendering of the environment is shown in (b). 
Each agent only has a field of range of $0.36$ meters as shown as the transparent blue circles.
Within $10$ seconds the agents under decentralized ergodic policies are able to provide a consensus on the location of the targets. 
(c) Extended Kalman filter values on the location of the target is shown. The black dashed line is the target location ground truth. 
We refer the reader to the multimedia \url{https://youtu.be/Jibt4GLj5sw} for more examples of target localization with a moving target.
}
\label{fig:target_local}
\end{figure*}

\subsection{Results}
Figure~\ref{fig:mapping_results} illustrates the algorithm for area coverage using a network of three decentralized robotic agents.
For comparison purposes, the area coverage of a single agent under the ergodic control policy is shown.
Due to sharing where each agent intends to go and where they have been, the outcome is a more efficient search as each agent chooses the best possible action that reduces the ergodic metric. 
The ergodic control automatically takes into account dynamic constraints and the histories of the other agents in order to allocate where each agent should go in a decentralized fashion.
We see this in Fig.~\ref{fig:mapping_results}(c) where the multi-agent system immediately acquires a good terrain model within the first ten seconds according the error norm on the estimate compared to what the single agent could be capable of accomplishing.

A comparison is presented in Fig.~\ref{fig:coverage_comparison} with respect to the centralized formulation of the algorithm.
Not much performance is lost within the first $5$ seconds of the algorithm when the robotic agents are still trying to achieve a consensus.
After each agent has fulfilled consensus, the decentralized ergodic policy functions minimize the ergodic metric comparably to the centralized version of the algorithm.

We further compare our algorithm with the work done in~\cite{kantaros2015distributed}. 
In~\cite{kantaros2015distributed}, the algorithm uses a visibility constraint which determined the location of the robot with linear dynamics based on the corridor.
We compare to this method using a very small visibility (only the point below the quadcopter) using the decentralized ergodic control scheme.
We present the area coverage problem in Fig.~\ref{fig:corridor_comparison} where we show the corridor used in~\cite{kantaros2015distributed} for area coverage using agents with nonlinear dynamics (quadcopter dynamics defined previously).
The initial positions of the agents were placed as closely as possible to~\cite{kantaros2015distributed}.
Since the visibility constraint is significantly small, this would require the robot to move in order to sufficiently cover the area.
As a result, the work in~\cite{kantaros2015distributed} would not be appropriate in a situation where the dynamics of the robot are needed to compensate for the sensor inefficiencies.
In contrast, our method compensates for the small visibility with motion as shown with the trajectories in Fig.~\ref{fig:corridor_comparison}.

We note in Fig.~\ref{fig:corridor_comparison} that the individual agents' respective ergodicity measures do poorly, whereas the ergodicity measure of the whole system does well.
This illustrates the efficacy of our method to coordinate the decentralized network to successfully minimize the ergodic objective.

\section{Decentralized Ergodic Control for Multi-Agent Target Localization}
\label{sec:ergodic-control-for-multi-agent-pursuit-evasion-games}
In this section, decentralized ergodic control is used for multi-agent target localization.
We use the example of multi-agent target localization because this platform provides us with novel demonstration of the decentralized ergodic control algorithm through a well known robotics problem.

\subsection{Problem Setup}
The goal of target localization is to have the agents locate the target (or targets) in the environment. 
Bearing only sensors~\cite{mavrommatiTRO2017realTime, deans2001experimental} are used for sensing the target with the same three agents as mentioned in Section~\ref{sec:terrain-mapping-using-ergodic-area-coverage} with quadcopter dynamics.
The obstacles are incorporated into the objective with an obstacle avoidance cost which we define by the function $\Theta(x) : \mathbb{R}^n \to \mathbb{R}^+$ which is a direct penalty if the agent goes near an obstacle.
In addition, we constrain the radius of the target sensor to $0.38$ meter diameter, thus limiting the total area coverage from the sensor.
The targets are uniformly dispersed throughout the terrain of size $[0,1] \times[0,1]$ \footnote{This is can be easily adjusted in experimentation if the terrain is much larger.} such that they do not intersect with the obstacles.
Targets are localized using an extended Kalman filter (EKF)~\cite{kalman1960new, julier1997new} with sensor noise assumed to be zero mean Gaussian with variance $\sigma^2 = 0.01$. 
The ergodic controller is initialized with a uniform target distribution. 
The prior on the targets is initialized as uniform over the search terrain and a distributed EKF updates the prior for the network system~\cite{carli2008distributed}.
The target distribution is given by the expected information density~\cite{miller2016ergodic, mavrommatiTRO2017realTime} 
\begin{equation*} \label{eq:eid}
\phi(s) = \eta \det \left[ \int_\theta  \frac{\partial \left( \Upsilon(\theta, s) \right) }{\partial \theta}^T \Sigma^{-1} \frac{\partial \left( \Upsilon(\theta, s) \right)}{\partial \theta} p(\theta) d\theta \right]
\end{equation*}
where $\eta$ is a normalization factor, $\Upsilon(\theta, s)$ is the bearing only measurement model parametrized by the position of the targets $\theta$.

\subsection{Results}
Figure~\ref{fig:target_local} illustrates trajectories of the 3 agents localizing 4 targets in the environment. 
In Fig.~\ref{fig:target_local} (a), each each agent chooses a different path that reduces the ergodic measure as well as increases the area coverage.
The agents each localize the targets within the first 10 seconds (as shown in Fig.~\ref{fig:target_local}(c)) while successfully avoiding obstacles (illustrated as the black colored squares).
Here, each agent is solving their own local control problem and only communicating the respective agent's $c_k$ values to the neighboring agents.
The resulting estimate error is within $0.001$ as specified by the Kalman filter and the measurement noise (zero mean Gaussian noise with variance $\sigma^2 = 0.01$).

We provide an additional example with a moving target in the attached multimedia \url{https://youtu.be/Jibt4GLj5sw}.
In this example, we use a particle filter to track the position of the target.
Note that each agent does not share the particle filter information with one another.
Instead, the $\phi_k$ values are also communicated which results in the agents converging over the single target.
As a result, the agents are able to hone in on the target even in the presence of obstacles.

\section{Conclusions}
\label{sec:conclusions}
We present a fully decentralized formulation of ergodic control for multi-agent systems with nonlinear dynamics.
A game theoretic analysis of the algorithm is provided showing the capabilities that the algorithm has on multi-agent games.
Examples of area coverage and target localization illustrate the flexibility of the algorithm for various multi-agent coordination tasks with nonlinear dynamics.
This opens up the possibilities of adapting other multi-agent objectives\textemdash such as pursuit-evasion games\textemdash into a decentralized network.




\balance

\bibliographystyle{IEEEtran}
\bibliography{references}

\end{document}